\documentclass[letterpaper, 10 pt, conference]{ieeeconf}  

\IEEEoverridecommandlockouts                              

\usepackage{times}

\usepackage{hyperref}%
\hypersetup{colorlinks=true, linkcolor=blue, breaklinks=true, urlcolor=blue}
\usepackage{doi}

\usepackage[all]{xy}
\usepackage{amsmath}
\usepackage{amscd}
\usepackage{mathrsfs}
\usepackage{amssymb}
\usepackage[utf8]{inputenc}
\usepackage[yyyymmdd]{datetime}
\usepackage{graphicx}
\usepackage{lipsum} 
\usepackage{multicol} \usepackage{mathtools}
\usepackage{multirow}
\usepackage{threeparttable}
\usepackage{stmaryrd}
\let\labelindent\relax
\usepackage{enumitem,sidecap}
\usepackage{wrapfig}

\usepackage[font=small]{caption}

\usepackage{multicol} \usepackage{mathtools}
\usepackage{version,xspace}
\usepackage{url,doi}
\newcommand{\until}[1]{\{1,\dots, #1\}}

\newcommand{\subscr}[2]{#1_{\textup{#2}}}
 \newcommand{\setdef}[2]{\{#1
	\; | \; #2\}}

\newcommand{\map}[3]{#1: #2 \rightarrow #3}

\newcommand\oprocendsymbol{\hbox{$\triangle$}}
\newcommand\oprocend{\relax\ifmmode\else\unskip\hfill\fi\oprocendsymbol}

\usepackage{enumitem} \renewcommand{\theenumi}{(\roman{enumi})}

\DeclareSymbolFont{bbold}{U}{bbold}{m}{n}
\DeclareSymbolFontAlphabet{\mathbbold}{bbold}
\DeclareMathAlphabet{\mymathbb}{U}{BOONDOX-ds}{m}{n}
\newcommand{\vect}[1]{\mathbbold{#1}}
\newcommand{\vectorones}[1][]{\mymathbb{1}_{#1}}

\newcommand{\real}{\mathbb{R}}

\newcommand{\realpositive}{\mathbb{R}_{>0}}

\newcommand{\seminorm}[1]{{\left\vert\kern-0.25ex\left\vert\kern-0.25ex\left\vert #1
		\right\vert\kern-0.25ex\right\vert\kern-0.25ex\right\vert}}

\newcommand{\semimeasure}[1]{\mu_{\seminorm{\cdot}}\kern-0.5ex\left(#1\right)}

\newcommand{\Lip}{\operatorname{Lip}}

\newcommand{\norm}[2]{\|#1\|_{#2}}

\DeclareMathOperator{\diag}{diag}

\renewcommand{\top}{\mathsf{T}} 

\newcommand{\OF}{\mathsf{F}}
\newcommand{\OG}{\mathsf{G}}

\newcommand{\ON}{\mathsf{N}}

\newcommand{\IFN}{\mathsf{IFN}}
\newcommand{\WFN}{\mathsf{WFN}}
\newcommand{\FN}{\mathsf{FN}}

\usepackage{tikz}
\usepackage{pgfplots}
\usetikzlibrary{decorations.markings}
\usetikzlibrary{patterns}
\usetikzlibrary{arrows}
\usetikzlibrary{shapes}
\usetikzlibrary{fit}
 \usetikzlibrary{calc}
\usetikzlibrary{decorations.pathreplacing}
\usetikzlibrary{arrows,positioning} 
\usetikzlibrary{pgfplots.groupplots}
\usetikzlibrary{backgrounds}

\newtheorem{theorem}{Theorem}[section]

\newtheorem{definition}[theorem]{Definition}
\newtheorem{remark}[theorem]{Remark}
\numberwithin{equation}{section}

\newcommand{\suchthat}{\;\ifnum\currentgrouptype=16 \middle\fi|\;}
\newcommand{\scirc}{\raise1pt\hbox{$\,\scriptstyle\circ\,$}}

\title{\LARGE \bf
Comparative Analysis of Interval Reachability for\\Robust Implicit and Feedforward Neural Networks 
}

\author{Alexander Davydov$^{1,*}$, Saber Jafarpour$^{2,*}$, Matthew Abate$^{2}$, Francesco Bullo$^{1}$, Samuel Coogan$^{2}$
\thanks{$^*$ These authors contributed equally and the order is alphabetical.}
\thanks{$^{1}$Alexander Davydov and Francesco Bullo are with the Center for Control, Dynamical Systems, and Computation, University of California, Santa Barbara,
         {\tt\small \{davydov, bullo\}@}ucsb.edu}%
\thanks{$^{2}$Saber Jafarpour, Matthew Abate, and Samuel Coogan are with the School of Electrical and Computer Engineering, Georgia Institute of Technology,
        {\tt\small \{saber, matt.abate, sam.coogan\}@}gatech.edu}%
}

\begin{document}

\maketitle
\thispagestyle{empty}
\pagestyle{empty}

\begin{abstract}
We use interval reachability analysis to obtain robustness guarantees for implicit neural networks (INNs). 
INNs are a class of implicit learning models that use implicit equations as layers and have been shown to exhibit several notable benefits over traditional deep neural networks.
We first establish that tight inclusion functions of neural networks, which provide the tightest rectangular over-approximation of an input-output map, lead to sharper robustness guarantees than the well-studied robustness measures of local Lipschitz constants.  
Like Lipschitz constants, tight inclusions functions are computationally challenging to obtain, and we thus propose using mixed monotonicity and contraction theory to obtain computationally efficient estimates of tight inclusion functions for INNs.
We show that our approach performs at least as well as, and generally better than, applying state-of-the-art interval bound propagation methods to INNs.
We design a novel optimization problem for training robust INNs
and we provide empirical evidence that suitably-trained INNs can be more robust than comparably-trained feedforward networks.
\end{abstract}


\section{Introduction}




Implicit neural networks (INNs) are a class of implicit learning models where the hidden layers are replaced with implicit equations~\cite{BA-JZK:17,SB-JZK-VK:19,LEG-FG-BT-AA-AYT:21}. 
Compared to their explicit counterparts, INNs are known to have advantages including (i) being more suitable for some problems such as constrained optimization problems~\cite{BA-JZK:17} (ii) being more memory efficient while maintaining comparable accuracy~\cite{SB-JZK-VK:19} (iii) allowing for new architecture possibilities~\cite{LEG-FG-BT-AA-AYT:21} (vi) showing improved training due to fewer vanishing and exploding gradients~\cite{AK-ZZ-VS:20}. 
Despite their benefits, INNs can suffer from well-posedness issues and convergence instabilities. Additionally, their input-output behavior may suffer from robustness issues and adversarial perturbations; indeed, such robustness vulnerabilities are a well-studied and major issue in traditional deep neural networks as well~\cite{CZ-WZ-IS-JB-DE-IG-RF:13}.

\paragraph*{Problem statement and motivation}

We aim to address two main challenges regarding robustness of neural networks: (i) designing algorithms to train provably robust INNs, and (ii) appropriately comparing the robustness of INNs with the robustness of their explicit counterparts. 


Obtaining provable robustness guarantees for learning algorithms has been a major goal in the machine learning literature~\cite{EW-ZK:18,AM-AM-LS-DT-AV:17}. For feedforward neural networks (FFNNs), three well-established methods for producing robustness guarantees include (i) Lipschitz bound approaches, (ii) interval bound propagation (IBP) methods, and (iii) convex-relaxation approaches. Lipschitz constants of neural networks are coarse but rigorous measures for their input-output robustness. While it has been shown that computing the exact Lipschitz constant of a neural network is NP-hard~\cite{AV-KS:18}, several efficient methods for providing sharp estimates of Lipschitz bounds are proposed in~\cite{PP-AK-JB-PK-FA:21-new,LPC-JCP:20}. 
IBP methods use interval analysis to provide box over-approximations of the reachable set of neural networks. These methods have been successfully used to train robust FFNNs~\cite{SG-etal:18,HZ-etal:20}. Convex-relaxation approaches are based on relaxations of nonlinear activation functions using either linear~\cite{CW-JZK:22} or quadratic constraints~\cite{MF-MM-GJP:20}. 

For robustness guarantees of INNs, several works provide estimates of their Lipschitz constants~\cite{CP-EW-JZK:21,SJ-AD-AVP-FB:21f,MR-RW-IRM:20-new}. However, global Lipschitz bounds 
do not provide information about the local sensitivity of the networks, for which local Lipschitz bounds are more informative but are harder to estimate. 
In~\cite{CW-JZK:22}, an iterative IBP approach for INNs is proposed, however, convergence of this iteration requires strong conditions which limit the expressivity of the resulting implicit models. In~\cite{TC-JBL-VM-EP:21}, an SDP-based method is proposed, however, this approach is computationally intensive and cannot be implemented in training.







\paragraph*{Contributions}

In this letter, we use interval reachability analysis to study robustness of INNs. We first introduce the notion of tight inclusion function associated to the INN that gives the tightest rectangular approximation for the input-output behavior of the INN. We show that tight inclusion functions are sharper than any robustness guarantees based on local Lipschitz bounds.  

Similar to Lipschitz constants, computing the tight inclusion function is computationally challenging. Instead, using mixed monotone systems theory and contraction theory, we provide computationally efficient estimates of the tight inclusion functions of INNs. 
Using two different interpretations of implicit neural networks, we compare our approach with the IBP approach for FFNNs. We show that our mixed monotone contracting approach is the natural extension of IBP methods to INNs and performs at least as well as, and generally better than, IBP methods naively applied to INNs. 


Lastly, we provide an algorithm to efficiently implement our mixed monotone contracting approach in the training optimization problem to design robust INNs. In numerical experiments, we compare the performance of INNs and FFNNs with a comparable number of parameters and demonstrate that suitably-trained INNs have improved certified and empirical robustness compared to their feedforward counterparts, even when trained with IBP. In the conference paper \cite{SJ-MA-AD-FB-SC:21y}, we focus on verifying robustness of INNs. In contrast, in this paper, we provide an efficient means for training robust INNs, compare theoretical robustness guarantees for INNs to FFNNs, extend~\cite[Theorem 1]{SJ-MA-AD-FB-SC:21y} and provide the proof for it, and present an empirical study of training for robustness.

\section{Mathematical Preliminary}
 For $x,y,z \in \real^n$, we write $x \leq y$ if $x_i \leq y_i$ for all $i \in \until{n}$ and $z \in [x,y]$ if $x \leq z \leq y$.
For $\eta\in \real^n_{>0}$, we define the diagonal matrix $[\eta]\in \real^{n\times n}$ by $[\eta]_{ii}=\eta_i$, for every $i\in \{1,\ldots,n\}$ and the diagonally weighted
$\ell_{\infty}$-norm by
$\|x\|_{\infty,[\eta]^{-1}}=\max_i|x_i|/\eta_i$, the diagonally weighted
$\ell_{\infty}$-matrix measure is defined by $\mu_{\infty,[\eta]^{-1}}(A) =\max_{i \in \until{n}} A_{ii} + \sum_{j \neq i} \frac{\eta_j}{\eta_i} |A_{ij}|$.
For any matrix $A\in \real^{n\times n}$, the spectral radius of $A$ is denoted by $\rho(A)$ and the elementwise absolute value of $A$ is denoted by $|A|\in \real^{n\times n}_{\ge 0}$. For two matrices $A,B$, let $A \otimes B$ denote their Kronecker product. Given a matrix
$B \in \mathbb{R}^{n\times m}$, we denote the non-negative part of $B$
by $[B]^+ := \max(B, 0)$ and the nonpositive part of $B$ by
$[B]^- := \min(B, 0)$. The Metzler and non-Metzler parts of a square matrix $A\in \real^{n\times n}$
are denoted by $\lceil A \rceil^{\mathrm{Mzl}}$ and
$\lfloor A \rfloor^{\mathrm{Mzl}}$, respectively, where
\begin{align*}
  (\lceil A \rceil^{\mathrm{Mzl}})_{ij} :=\begin{cases}
    A_{ij} & \text{if } A_{ij} \geq 0\; \mbox{or} \; i = j\\
    0 & \mbox{otherwise,}
  \end{cases}
\end{align*}
and $\lfloor A\rfloor^{\mathrm{Mzl}} := A-\lceil A \rceil^{\mathrm{Mzl}}$. The subset $\mathcal{T}^n\subset \real^{2n}$ is defined by $\mathcal{T}^{n} := \setdef{(x,\widehat{x})\in \real^{2n}}{x\le \widehat{x}}$. Given a map $f:\real^n\to \real^m$, a set $\mathcal{U}\subset \real^n$, and $p\in [1,\infty]$, the $\ell_p$-Lipschitz constant of $f$ on $\mathcal{U}$ is the smallest $\Lip^{\mathcal{U}}_p(f)\in \real_{\ge 0}$ such that
    $\|f(x)-f(y)\|_p \le \Lip^{\mathcal{U}}_p(f)\|x-y\|_p$, for all $x,y\in \mathcal{U}$.
Given a vector-valued map $f:\real^n\to \real^n$ and $\alpha\in [0,1]$, the $\alpha$-average map $f_{\alpha}:\real^n\to \real^n$ is defined by $f_{\alpha}(x)=(1-\alpha)x + \alpha f(x)$, for every $x\in \real^n$.

\section{Inclusion Functions}

Given a mapping $f$, an $\ell_{\infty}$-box over-approximation of the image of $f$ is attainable via an inclusion function. 

\begin{definition}[Inclusion function]\label{def:inclusion}
Let $f:\real^r\to \real^q$ be a mapping. Then $\OF=
\left[\begin{smallmatrix}\underline{\OF}\\ \overline{\OF}\end{smallmatrix}\right]:\mathcal{T}^r\to \real^{2q}$ is  an inclusion function for $f$, if, for every $x\le y \le \widehat{x}$, 
\begin{enumerate}
    \item  $\underline{\OF}(y,y)\ge \underline{\OF}(x,\widehat{x})$ and $\overline{\OF}(y,y)\le \overline{\OF}(x,\widehat{x})$; 
    \item $\underline{\OF}(x,x)=\overline{\OF}(x,x)=f(x)$. 
\end{enumerate}
Moreover, the inclusion function $\OF$ for $f$ is called \emph{tight}, if
\begin{enumerate}\setcounter{enumi}{2}
    \item for every inclusion function $\OG= 
    \left[\begin{smallmatrix}\underline{\OG}\\ \overline{\OG}\end{smallmatrix}\right]:\mathcal{T}^r\to \real^{2q}$ of $f$, we have
    \begin{align*}
        \underline{\OG}(x,\widehat{x})\le \underline{\OF}(x,\widehat{x}),\quad
        \overline{\OF}(x,\widehat{x})\ge \overline{\OG}(x,\widehat{x}),\quad\mbox{for all  }x\le \widehat{x}
    \end{align*} 
\end{enumerate}
\end{definition}
If $\OF$ is an inclusion function for $f$, then it is easy to see that,
\begin{align}\label{eq:interval-bound}
    f([x,\widehat{x}])\subseteq [\underline{\OF}(x,\widehat{x}),\overline{\OF}(x,\widehat{x})],\quad \mbox{for all }x\le \widehat{x}.
\end{align}
 If an inclusion function $\OF$ for $f$ is tight, then it provides the tightest interval over-approximation as in~\eqref{eq:interval-bound} compared to all the other inclusion functions for $f$. Given a map $f:\real^r\to \real^q$, one can use~\cite[Theorem 1]{LY-NO:19} to compute the tight inclusion function $\OF =\left[\begin{smallmatrix}\underline{\OF}\\ \overline{\OF}\end{smallmatrix}\right]$ for $f$, component-wise. Indeed, for every $i\in \{1,\ldots,n\}$, one can show that:
\begin{align}\label{eq:tight-inclusion}
    \underline{\OF}_i(x,\widehat{x})= \min_{z\in [x,\widehat{x}]} f_i(z),\quad
    \overline{\OF}_i(x,\widehat{x})=\max_{z\in [x, \widehat{x}]} f_i(z)
\end{align} 
 The next Theorem studies the connection between the local Lipschitz constants and the tight inclusion functions.

\begin{theorem}[Inclusion function vs. Lipschitz constant]\label{thm:LipInc}
Let $f:\real^r\to \real^q$ be a continuous mapping and $\OF=\left[\begin{smallmatrix}\underline{\OF}\\ \overline{\OF}\end{smallmatrix}\right]:\mathcal{T}^r\to \real^{2q}$ be the tight inclusion function for $f$. Then, for every $\underline{x}\le \overline{x}$, we have 
\begin{equation*}
    \|\overline{\OF}(\underline{x},\overline{x})-\underline{\OF}(\underline{x},\overline{x})\|_{\infty} \le \Lip^{[\underline{x},\overline{x}]}_{\infty}(f) \|\underline{x}-\overline{x}\|_{\infty}.
\end{equation*}
\end{theorem}
\begin{proof} Let $i\in \{1,\ldots,k\}$ be such that $\|\overline{\OF}(\underline{x},\overline{x})-\underline{\OF}(\underline{x},\overline{x})\|_{\infty} = \left|\overline{\OF}_i(\underline{x},\overline{x})-\underline{\OF}_i(\underline{x},\overline{x})\right|$. 
Note that since $f$ is continuous and the box $[\underline{x},\overline{x}]$ is compact, there exist $\eta^*,\xi^*\in [\underline{x},\overline{x}]$ such that 
    $\max_{y\in [\underline{x},\overline{x}]}f_i(y) =f_i(\eta^*), \min_{y\in [\underline{x},\overline{x}]}f_i(y)=f_i(\xi^*).$
This implies that $\|\overline{\OF}(\underline{x},\overline{x})-\underline{\OF}(\underline{x},\overline{x})\|_{\infty}= |f_i(\eta^*)-f_i(\xi^*)| \le  \|f(\xi^*)-f(\eta^*)\|_{\infty} \le 
   \Lip^{[\underline{x},\overline{x}]}_{\infty}\|\xi^*-\eta^*\|_{\infty} \le \Lip^{[\underline{x},\overline{x}]}_{\infty}(f) \|\underline{x}-\overline{x}\|_{\infty}$. 
\end{proof}

\begin{remark}[Tight inclusion functions] {\ } 
\begin{enumerate}
    \item In general, 
    finding tight inclusion functions using~\eqref{eq:tight-inclusion} is not computationally tractable. This motivates developing efficient methods for estimating the tight inclusion function. 
    \item Theorem~\ref{thm:LipInc} shows that a tight inclusion function for $y=f(x)$ will provide a tighter over-approximation of the image of $f$ than is attainable from local Lipschitz constants of $f$. 
\end{enumerate}
\end{remark}

  \section{Implicit Neural Networks}
  We consider the implicit neural network 
  \begin{align}\label{eq:INN}
  z&=\sigma(Wz+Ux+b):=\ON(z,x),\nonumber\\
  y&= Cz+c,
  \end{align}
  where $z\in \real^n$ is the hidden variable, $x\in \real^r$ is the input, $y\in \real^{q}$ is the output, $W\in \real^{n\times n}$, $U\in \real^{n\times r}$, and $C\in \real^{q\times n}$ are the weight matrices and $b\in \real^n$ and $c\in \real^q$ are bias vectors. Moreover, $\sigma$ is a diagonal activation function (e.g., ReLU) defined by $\sigma(z_1,\ldots,z_n)=(\sigma_1(z_1),\ldots,\sigma_n(z_n))^{\top}$, where for every $i\in \{1,\ldots,n\}$, the activation function  $\sigma_i:\real\to \real$ satisfies $0\le \frac{\sigma_i(x)-\sigma_i(y)}{x-y}\le 1$ for all $x\neq y \in \real$. Compared to traditional neural networks, INNs replace the layers with a fixed-point equation. This change in the structure is known to allow for new architecture possibilities and provide alternative approaches to deep modeling.   

   \paragraph*{Generalized architecture} Notably, FFNNs can be considered as special cases of INNs~\cite{LEG-FG-BT-AA-AYT:21}. Consider the FFNN 
 \begin{align}\label{eq:FFNN}
     z^{i} &= \sigma(W_i z^{i-1} + b_i)=:\FN_i(z^{i-1}), \qquad i\in \{1,\ldots,k\}, \nonumber\\
     y &= C z^k + c
 \end{align}
 where $z_0=x\in \real^r$ is the input. For every $i\in \{1,\ldots,k\}$, $z^i\in \real^{n_i}$, $W_i\in \real^{n_{i-1}\times n_{i}}$, and $b_i\in \real^{n_{i-1}}$ are the weights, the biases, and the hidden variables in the $i$-th layer of the network, respectively. Finally, $y\in \real^q$ is the output, $C$ and $c$ are the output layer's weight matrix and bias vector. The FFNN~\eqref{eq:FFNN} is equivalent to the following INN
 \begin{align}\label{eq:associated-INN}
     z &= \sigma (W^{\FN} z + U^{\FN}x + b)=: \IFN(z,x),\nonumber \\  y &= C^{\ON}z + c
 \end{align}
 where $z=[z_k,\ldots,z_1]^{\top}$, $b=[b_k,\ldots,b_1]^{\top}$, and $W^{\FN}$, $U^{\FN}$, and $C^{\FN}$ are defined as follows:
 \begin{align}\label{eq:def-WFN}
    W^{\FN} &= \begin{bmatrix}
     \vect{0} & W_k & \vect{0} & \cdots & \vect{0} \\ 
     \vect{0} & \vect{0} & W_{k-1} & \cdots & \vect{0}\\
     \vdots & \vdots & \vdots &\ddots & \vdots\\
     \vect{0} & \vect{0} & \vect{0} & \cdots & W_1 \\
     \vect{0} & \vect{0} & \vect{0} & \cdots & \vect{0}
    \end{bmatrix}, \quad U^{\FN} = \begin{bmatrix}
     \vect{0}\\ \vect{0} \\ \vdots \\ \vect{0} \\ W_{0} 
    \end{bmatrix}\nonumber\\
    C^{\FN} &= \begin{bmatrix}
     \vect{0} & \vect{0} & \vect{0} & \cdots & C
    \end{bmatrix}.
    \end{align}
    Using this perspective, implicit neural networks generalize FFNNs by allowing arbitrary interconnections between layers leading to full weight matrices $W$, $U$, and $C$.
    
    \paragraph*{Alternative deep modeling}
    By replacing the notion of layer with an algebraic equation, implicit neural networks provide a novel perspective toward deep modeling. Consider the class of 
    FFNN where weights and biases are equal for each layer (i.e., the network is weight-tied) and the input is injected to each layer  given by 
\begin{align}\label{eq:weight-tied}
    z^{i} &= \sigma(Wz^{i-1} + Ux + b):=\WFN(z^{i-1},x), \; i \in \until{k},\nonumber\\
    y &= Cz_k + c
\end{align}
where $z^0 = x$. While weight-tying may appear restrictive, it is usually considered as a form of regularization that stabilizes training and significantly reduces the model size~\cite{SB-JZK-VK:19}. If the depth of the network increases, i.e., $k\to \infty$ and the  iteration~\eqref{eq:weight-tied} converges, then the weight-tied input-injected neural network~\eqref{eq:weight-tied} is equivalent to the implicit neural network~\eqref{eq:INN}. Using this perspective, INNs provide a depth-independent alternative to deep FFNNs. 



Suppose that, for every input $x\in \real^r$, the implicit neural network~\eqref{eq:INN} has a unique fixed point $z^*(x)\in \real^n$. Then, the  \emph{input-output map} $f:\real^r\to \real^q$ is given by 
\begin{align}\label{eq:INN-input-output}
f(x) := y = Cz^*(x) + c. 
\end{align}
In the next section, our goal is to provide estimates for the tight inclusion function of the input-output map $f$.

\section{Reachability Analysis of Implicit Neural Networks}\label{sec:reachability}

In this section, we use mixed monotone system theory to present a framework for estimating the tight inclusion function for the input-output map of INNs. Given an implicit neural network~\eqref{eq:INN} and input bounds $\underline{x}\leq \overline{x}\in \real^r$, we first introduce the embedding map $\ON^{\mathrm{E}}:\real^{2n}\times \real^{2r}\to \real^n$ by
      \begin{multline*}
         \ON^{\mathrm{E}}(\underline{z},\overline{z},\underline{x},\overline{x})=\sigma(\lceil W \rceil^{\mathrm{Mzl}} \underline{z}+\lfloor W \rfloor^{\mathrm{Mzl}} \overline{z} + \\ 
          [U]^{+}\underline{x} + [U]^{-}\overline{x} + b).
      \end{multline*}
Using the embedding map $\ON^{\mathrm{E}}$, we define the \emph{embedded implicit neural network} associated to~\eqref{eq:INN} by 
\begin{align}\label{eq:INN-embedding}
        \begin{bmatrix}\underline{z}\\\overline{z}\end{bmatrix}
        = \begin{bmatrix}\ON^{\mathrm{E}}(\underline{z},\overline{z},\underline{x},\overline{x})\\ \ON^{\mathrm{E}}(\overline{z},\underline{z},\overline{x},\underline{x})\end{bmatrix}, \;\;
        \begin{bmatrix}\underline{y}\\\overline{y}\end{bmatrix} =
        \begin{bmatrix}[C]^+ & [C]^- \\ [C]^- & [C]^+\end{bmatrix}  \begin{bmatrix}\underline{z}\\\overline{z}\end{bmatrix} +
        \begin{bmatrix}c\\ c\end{bmatrix}. 
      \end{align}
      The embedded INN~\eqref{eq:INN-embedding} can be considered as an INN with the box input $[\underline{x}, \overline{x}]$ and the box output $[\underline{y},\overline{y}]$. In the next theorem, we use the embedded system~\eqref{eq:INN-embedding} to obtain an inclusion function for the input-output map~\eqref{eq:INN-input-output}.
      
  \begin{theorem}[Inclusion function via embedded network]\label{thm:L4DC}
Consider the implicit neural network~\eqref{eq:INN} and its associated embedded implicit neural network~\eqref{eq:INN-embedding}. Suppose that there exists $\eta\in \real^n_{>0}$ such that $\mu_{\infty,[\eta]^{-1}}(W)<1$. For every $\underline{x}\le x\le \overline{x}$, and every $\alpha\in (0, \alpha^*:=[1-\min_{i\in \{1,\ldots,n\}}[W_{ii}]^{-}]^{-1}]$, the following statements hold:
\begin{enumerate}
          \item\label{p1:L4DC} the iterations
          $
          \left[\begin{smallmatrix}\underline{z}^{k+1} \\ \overline{z}^{k+1}\end{smallmatrix}\right] = 
          \left[\begin{smallmatrix}\ON^{\mathrm{E}}_{\alpha}(\underline{z}^k,\overline{z}^k,\underline{x},\overline{x})\\
          \ON^{\mathrm{E}}_{\alpha}(\overline{z}^k,\underline{z}^k,\overline{x},\underline{x})\end{smallmatrix}\right]$ are contracting with respect to the norm  $\|\cdot\|_{\infty,I_2\otimes [\eta]^{-1}}$ and converge to the unique fixed-point $\begin{bmatrix}
          \underline{z}^*\\
           \overline{z}^*
          \end{bmatrix}$ of the embedded INN~\eqref{eq:INN-embedding}; 
          \item\label{p2:L4DC} the iterations $z^{k+1}=\ON_{\alpha}(z^k,x)$ are contracting with respect to the norm $\|\cdot\|_{\infty,[\eta]^{-1}}$ and converges to the unique fixed point $z^*\in [\underline{z}^*,\overline{z}^*]$ of the INN~\eqref{eq:INN};
          \item\label{p3:L4DC} the map $\OF^{\ON}=\left[\begin{smallmatrix}
            \underline{\OF}^{\ON}\\
             \overline{\OF}^{\ON}
          \end{smallmatrix}\right]:\mathcal{T}^{r}\to \real^{2q}$ defined by 
          \begin{align}\label{eq:inclusion-INN}
              \underline{\OF}^{\ON}(\underline{x},\overline{x}) &= [C]^{+}\underline{z}^*+[C]^{-}\overline{z}^*+c\nonumber\\
               \overline{\OF}^{\ON}(\underline{x},\overline{x}) &= [C]^{+}\overline{z}^*+[C]^{-}\underline{z}^*+c
          \end{align}
          is an inclusion function for $f$ defined in~\eqref{eq:INN-input-output}.
\end{enumerate}
  \end{theorem}
\begin{remark}[Mixed monotone contracting approach]
  Theorem~\ref{thm:L4DC} can be interpreted as a dynamical
      system approach to study robustness of INNs. Indeed, the $\alpha$-average
      iteration in part~\ref{p2:L4DC} is the forward Euler
      discretization of the dynamical system $\frac{dz}{dt}=-z+\ON(z,x)$. The convergence of the iterations is due to the contraction property of the dynamical system and the estimate for the inclusion function is due to the mixed monotonicity of the dynamical system associated with the embedded INN~\cite{SJ-MA-AD-FB-SC:21y}. 
\end{remark}
  \begin{proof}
Regarding part~\ref{p1:L4DC},  we define $\tilde{\sigma} = I_2\otimes \sigma$, the map $\OG:\real^{2n}\to \real^{2n}$ by $\OG(\underline{z},\overline{z}) = \left[\begin{smallmatrix}\lceil W \rceil^{\mathrm{Mzl}} \underline{z} +\lfloor W \rfloor^{\mathrm{Mzl}}\overline{z}\\
          \lfloor W \rfloor^{\mathrm{Mzl}}\underline{z} +
          \lceil W
            \rceil^{\mathrm{Mzl}}\overline{z}\end{smallmatrix}\right]$ and the matrices $D = \left[\begin{smallmatrix}[U]^{+} & [U]^{-} \\ [U]^{-} & [U]^{+}\end{smallmatrix}\right]$ and $w=[\underline{x},\overline{x}]^\top$. Then define $ \tilde{\sigma}^{\OG}: \real^{2n}\to \real^{2n}$ as follows
            \begin{align*}
             \tilde{\sigma}^{\OG}(\underline{z},\overline{z},w) :=   \begin{bmatrix}
           \ON^{\mathrm{E}}(\underline{z},\overline{z},\underline{x},\overline{x})\\
          \ON^{\mathrm{E}}(\overline{z},\underline{z},\overline{x},\underline{x})
          \end{bmatrix} = \tilde{\sigma}(\OG(\underline{z},\overline{z})+Dw+I_2\otimes b).
            \end{align*}
             The assumptions on each
        scalar activation function imply that (i)
        $\map{\tilde{\sigma}}{\real^{2n}}{\real^{2n}}$ is non-expansive with respect
        to $\|\cdot\| := \norm{\cdot}{\infty,I_2\otimes [\eta]^{-1}}$ and (ii) for every
        $p,q \in \real$, there exists $\theta_i \in [0,1]$ such that
        $\sigma_i(p) - \sigma_i(q) = \theta_i(p - q)$ or in the matrix
        form $\tilde{\sigma}(\mathbf{p})-\tilde{\sigma}(\mathbf{q}) = \Theta (\mathbf{p}-\mathbf{q})$ where $\Theta\in \real^{2n\times 2n}$ is a
        diagonal matrix with diagonal elements $\theta_i\in
        [0,1]$ and $\mathbf{p},\mathbf{q}\in \real^{2n}$. As a result, for every $y_1,y_2\in \real^{2n}$, we have
         \begin{multline*}
          \|\tilde{\sigma}^{\OG}_{\alpha}(y_1,w)-\tilde{\sigma}^{\OG}_{\alpha}(y_2,w)\|
          \\ = \|(1-\alpha)(y_1-y_2) + \alpha \Theta
            (\OG(y_1)-\OG(y_2))\|
            \\  \le \sup_{y\in \real^{2n}}\|(1-\alpha)I_{2n} + \alpha
                                                            \Theta D \OG(y)\|
                                                            \|y_1-y_2\|
        \end{multline*}
         where the inequality holds by the mean value theorem. Then, for
        every $\alpha \in (0,[1-\min_{i}\inf_{y\in \real^{2n}}
          (\Theta D \OG (y))_{ii}]^{-1}]$,
        \begin{align*}
          \|I_{2n} + &\alpha(-I_{2n}+\Theta D\OG(y))\|
          \\ & = 1 +
          \alpha \mu_{\infty,I_2\otimes[\eta]^{-1}}\big(-I_{2n}+\Theta D\OG(y)\big) \\
          &
          =
          1
          +
          \alpha
          \big(-1
          +
          \mu_{\infty,I_2\otimes[\eta]^{-1}}(\Theta D\OG(y))\big)\\ &
          \le
          1
          +
          \alpha \big(-1+\mu_{\infty,I_2\otimes [\eta]^{-1}}(D\OG(y))^{+}\big) \\ & \le 1 -
          \alpha (1-\mu_{\infty,[\eta]^{-1}}(W)^{+}) < 1, 
        \end{align*} 
        where the first equality holds by~\cite[Lemma 7(i)]{SJ-AD-AVP-FB:21f}, the second
        equality holds by translation property of matrix
        measures, the third
        inequality holds by~\cite[Lemma 8(i)]{SJ-AD-AVP-FB:21f}, and the fourth inequality holds by the definition of matrix measure. Moreover, since $\theta_i\in
        [0,1]$, we have $\theta_i (D\OG)_{ii} \ge (D\OG)_{ii}^-$, for every
        $i\in \{1,\ldots,2n\}$. This means that 
        \begin{align*}
          \inf_{y\in \real^{2n}}
          (\Theta D \OG (y))_{ii} \ge  \inf_{y\in \real^{2n}}
          (D\OG_{ii}(y))^- = (W_{ii})^-. 
          \end{align*}
        This implies that,  for
        every $\alpha \in (0,\alpha^*]$,
        \begin{multline*}
          \|\tilde{\sigma}^{\OG}_{\alpha}(x_1,u)-\tilde{\sigma}^{\OG}_{\alpha}(x_2,u)\|
          \\ \le
          (1 -\alpha (1-\mu_{\infty,[\eta]^{-1}}(W)^{+})) \|x_1-x_2\|
        \end{multline*}
        Since $1 -\alpha (1-\mu_{\infty,[\eta]^{-1}}(W)^{+}) < 1$,
        $\tilde{\sigma}^{\OG}_{\alpha}(\cdot,w)$ is a contraction mapping with respect to $\|\cdot\|_{\infty,I_2\otimes [\eta]^{-1}}$ for every
        $\alpha \in (0,\alpha^*]$. It is easy to see that $\tilde{\sigma}^{\OG}_{\alpha}$ and $\tilde{\sigma}^{\OG}$ have the identical fixed-points, for every $\alpha\in [0,1]$. Therefore the iterations in part~\ref{p1:L4DC} converge to the unique fixed point of the embedded INN~\eqref{eq:INN-embedding}. Regarding part~\ref{p2:L4DC}, the proof follows by applying the same argument as in the proof of part~\ref{p1:L4DC} and using $\sigma(Wz+Ux+b)$ instead of $\tilde{\sigma}^{\OG}(\underline{z},\overline{z},\underline{x},\overline{x})$. Now, we show that $\underline{z}^*\le z^*\le \overline{z}^*$. We choose the initial condition $\left[\begin{smallmatrix}\underline{z}^0 \\ \overline{z}^0\end{smallmatrix}\right]$ for the iterations in part~\ref{p1:L4DC} and choose an initial condition $z^{0}\in \real^n$ satisfying $\underline{z}^0\le z^{0} \le \overline{z}^0$ for the iterations in part~\ref{p2:L4DC}. We prove by induction that, for every $k\in \mathbb{Z}_{\ge 0}$, we have $\underline{z}^{k}\le z^{k}\le \overline{z}^{k}$. Suppose that this claim is true for $k\in \{1,\ldots,m\}$ and we show that this claim is true for $k=m+1$. We first define $p=\lceil W \rceil^{\mathrm{Mzl}} \underline{z}^m +\lfloor W \rfloor^{\mathrm{Mzl}}\overline{z}^m + [U]^+\underline{x} + [U]^-\overline{x} +b $ and $q=W z^m +Ux + b$. Then we have
        \begin{align*}
         \underline{z}^{m+1}-z^{m+1} & =(1-\alpha^*)(\underline{z}^{m}-z^m) + \alpha^* (\sigma(p) - \sigma(q))\\ &= \left((1-\alpha^*)I_n + \alpha^*\Theta \lceil W \rceil^{\mathrm{Mzl}}\right)(\underline{z}^{m}-z^m) \\ & + \alpha^*\Theta \lfloor W \rfloor^{\mathrm{Mzl}} (\overline{z}^m-z^m) \\ &+ \alpha^*\Theta[U]^{+}(\underline{x}-x) + \alpha^*\Theta[U]^-(\overline{x}-x), 
        \end{align*}
        where the non-negative diagonal matrix $\Theta =\diag(\theta_i)\in \real^n$ is defined as follows: for every $i\in \{1,\ldots,n\}$, $\theta_i\in [0,1]$ is such that $\sigma_i(p_i)-\sigma_i(q_i) = \theta_i(p_i-q_i)$. Moreover, we know that $\Theta \lfloor W \rfloor^{\mathrm{Mzl}}\le \vect{0}_{n\times n}$ and, for every $i\in \{1,\ldots,n\}$,
        \begin{align*}
            (1-\alpha^*) + \alpha^*\theta_i W_{ii} \ge (1-\alpha^*) + \alpha^* W_{ii}^{-} \ge 0.
        \end{align*}
        This implies that $(1-\alpha^*)I_n + \alpha^*\Theta \lceil W \rceil^{\mathrm{Mzl}}\ge \vect{0}_{n\times n}$. Additionally, we have $\Theta [U]^{+}\ge \vect{0}_{n\times r}$ and $\Theta [U]^{-}\le \vect{0}_{n\times r}$. Therefore, using the induction assumption, we get $\underline{z}^{m+1}-z^{m+1}\le \vect{0}_n$. Similarly, one can show that $z^{m+1}-\overline{z}^{m+1}\le \vect{0}_n$. 
        As a consequence,
            $\underline{z}^* = \lim_{k\to \infty}\underline{z}^{k} \le \lim_{k\to \infty} z^k =z^* \le \lim_{k\to \infty}\overline{z}^{k}= \overline{z}^*.$
        This proves part~\ref{p2:L4DC}. The proof of part~\ref{p3:L4DC} follows easily from parts~\ref{p1:L4DC} and~\ref{p2:L4DC} and by checking the properties of inclusion functions from Definition~\ref{def:inclusion}. 
        \end{proof}


\section{Feedforward vs. implicit neural networks }

In this section, we compare the robust training framework developed in Sections~\eqref{sec:reachability} and~\eqref{sec:train} with the IBP approach developed in~\cite{SG-etal:18}. Consider a $k$-layer FFNN with input-output map $f(x)=:y$ defined by~\eqref{eq:FFNN}.
For every $i\in \{1,\ldots,k\}$, following~\cite{SG-etal:18,HZ-etal:20}, we can obtain layer-wise upper and lower bounds for the hidden variables as follows: 
 \begin{align*}
     \underline{z}^{i+1}:=\FN_i^{\mathrm{E}}(\underline{z}^i,\overline{z}^i) = \sigma( [W_i]^+ \underline{z}^i+[W_i]^- \overline{z}^i + b_i),\nonumber\\
     \overline{z}^{i+1}:=\FN_i^{\mathrm{E}}(\overline{z}^i,\underline{z}^i) = \sigma( [W_i]^+ \overline{z}^i+[W_i]^- \underline{z}^i + b_i).
 \end{align*}
By applying this bounding technique recursively, one can obtain the upper bound $\overline{y}$ and lower bound $\underline{y}$ for the output of the FFNN. The IBP inclusion function $\OF^{\FN} =
      \left[\begin{smallmatrix}\underline{\OF}^{\FN}\\ \overline{\OF}^{\FN}\end{smallmatrix}\right]:\mathcal{T}^r\to \real^{2q}$ for the input-output map $f$ is then defined by:
      \begin{align}\label{eq:inclusion:FFNN}
       \underline{\OF}^{\FN}(\underline{x},\overline{x}) = \underline{y},\qquad  \overline{\OF}^{\FN}(\underline{x},\overline{x}) = \overline{y},   
      \end{align}
      In the next two subsections, we use the two perspective \emph{generalized architecture} and \emph{alternative deep modeling} toward INNs, to establish connections between the IBP approach in~\cite{SG-etal:18} and our mixed monotone contracting approach.

  \subsection{Generalized architecture}
  By considering finite-depth FFNNs as a special case of implicit neural networks, one can show that our mixed monotone contracting approach is a generalization of the IBP approach in~\cite{SG-etal:18} to INNs.

  \begin{theorem}[Embedded feedforward neural networks]\label{thm:L4DC-FFNN}
Consider the FFNN~\eqref{eq:FFNN} with the associated implicit neural network~\eqref{eq:associated-INN}. The following statements hold: 
\begin{enumerate}
\item \label{p3:L4DC-FFNN} for every $i\!\in \!\{1,\ldots,k\}$, the function $(\underline{z},\overline{z})\!\mapsto\!\!
\left[\begin{smallmatrix}\FN_i^{\mathrm{E}}(\underline{z},\overline{z}) \\ \FN_i^{\mathrm{E}}(\overline{z},\underline{z})\end{smallmatrix}\right]$ is a tight inclusion function for the $i$th layer evaluation map $\FN_i(z) := \sigma(W_i z+ b_i)$;
\item\label{p0:L4DC-FFNN} there exists $\eta\in \real_{>0}^n$ such that $\mu_{\infty,[\eta]^{-1}}(W^{\FN}) < 1$.
    \end{enumerate}
If $\OF^{\IFN}$ is the inclusion function obtained from~\eqref{eq:associated-INN} using mixed monotone contracting approach and $\OF^{\FN}$ is the inclusion function~\eqref{eq:inclusion:FFNN} obtained using IBP approach, then 
\begin{enumerate}\setcounter{enumi}{2}
          \item\label{p4:L4DC-FFNN} $\OF^{\IFN}(\underline{x},\overline{x}) = \OF^\FN (\underline{x},\overline{x})$, for every $\underline{x}\le \overline{x}\in \real^r$.
          \end{enumerate}
  \end{theorem}
  \begin{proof}
  Regarding part~\ref{p3:L4DC-FFNN}, we solve the optimization problems~\eqref{eq:tight-inclusion} directly to obtain the tight inclusion function for $\FN_i$. Suppose that $W_i\in \real^{n_{i-1} \times n_i}$. For every $j\in \{1,\ldots,n_i\}$, we have
  \begin{align*}
  \min_{z\in [\underline{z},\overline{z}]} & [\sigma(W_iz+b)]_j = \sigma_j(\min_{z\in [\underline{z},\overline{z}]}(W_iz+b)_j) \\ & = [\sigma([W_i]^+\underline{z} + [W_i]^-\overline{z}+ b)]_j = [\FN^{\mathrm{E}}_i(\underline{z},\overline{z})]_j,
  \end{align*}
  where the first equality holds because $\sigma_j$ is weakly-increasing and the second equality holds by choosing the optimal solution $z^* = [z^*_1,\ldots,z^*_{n_i}]^{\top}\in [\underline{z},\overline{z}]$ defined by
  \begin{align*}
      z^*_l = \begin{cases}
      \overline{z}_l & [W_i]^+_{jl} = 0,\\
      \underline{z}_l & \mbox{otherwise},
      \end{cases}
  \end{align*}
  for every $l\in \{1,\ldots,n_i\}$. Similarly, one can show that $\max_{z\in [\underline{z},\overline{z}]} [\sigma(W_i z+b)]_j =[\FN^{\mathrm{E}}_i(\overline{z},\underline{z})]_j$, for every $j\in \{1,\ldots,n_i\}$. This implies that $(\underline{z},\overline{z})\mapsto \left[\begin{smallmatrix} \FN^{\mathrm{E}}_i(\underline{z},\overline{z})\\ \FN^{\mathrm{E}}_i(\overline{z},\underline{z}) \end{smallmatrix}\right]$ is a tight inclusion function for $z \mapsto \FN_i(z)$. 
  
  Regarding part~\ref{p0:L4DC-FFNN}, we choose $\eta = [\delta, \dots, \delta^{k}]\in \real^n_{>0}$ for $\delta > 0$. Then, using the definition of $W^{\FN}$ in~\eqref{eq:def-WFN}, we get 
  \begin{align*}
      \mu_{\infty,[\eta]^{-1}}(W^{\FN}) & = \mu_{\infty,[\eta]^{-1}}\left(\begin{bmatrix}
     \vect{0} & W_k  & \cdots & \vect{0} \\ 
     \vdots  & \vdots &\ddots & \vdots\\
     \vect{0} & \vect{0} & \cdots & W_1 \\
     \vect{0} & \vect{0} & \cdots & \vect{0}
    \end{bmatrix}\right) \\ & = \mu_{\infty}\left(\begin{bmatrix}
     \vect{0} & \delta^{-1} W_k  & \cdots & \vect{0} \\ 
     \vdots  & \vdots &\ddots & \vdots\\
     \vect{0} & \vect{0} & \cdots & \delta^{-1} W_1 \\
     \vect{0} & \vect{0} & \cdots & \vect{0}
    \end{bmatrix}\right) \\ & = \delta^{-1} \max_{i\in\{1,\ldots,k\}} \{\|W_i\|_{\infty}\}. 
  \end{align*}
  By choosing $\delta>0$ such that $\max_{i\in\{1,\ldots,k\}}\{\|W_i\|_{\infty}\} < \delta$, we get $\mu_{\infty,[\eta]^{-1}}(W^{\FN}) < 1$. 
  
  Regarding part~\ref{p4:L4DC-FFNN}, by part~\ref{p0:L4DC-FFNN} and Theorem~\ref{thm:L4DC}\ref{p3:L4DC}, for the implicit neural network~\eqref{eq:associated-INN}, the inclusion function $\OF^{\IFN} = \left[\begin{smallmatrix}
            \underline{\OF}^{\IFN}\\
             \overline{\OF}^{\IFN}
          \end{smallmatrix}\right]$ is well-defined and is given by
  \begin{align*}
                  \underline{\OF}^{\IFN}(\underline{x},\overline{x}) &= [C]^{+}\underline{z}^*+[C]^{-}\overline{z}^*+c,\nonumber\\
               \overline{\OF}^{\IFN}(\underline{x},\overline{x}) &= [C]^{+}\overline{z}^*+[C]^{-}\underline{z}^*+c,
  \end{align*}
  where $\left[ \begin{smallmatrix}\underline{z}^* \\ \overline{z}^* \end{smallmatrix}\right] = \left[ \begin{smallmatrix}(\underline{z}^*_1,\ldots, \underline{z}^*_k)^{\top}\\  (\overline{z}^*_1,\ldots, \overline{z}^*_k)^{\top}\end{smallmatrix} \right]$ is the unique solution of the following fixed-point equation 
  \begin{align}\label{eq:fixed-point-IFN}
      \begin{bmatrix}
       \underline{z}\\
       \overline{z}
      \end{bmatrix} = \begin{bmatrix}
       \IFN^{\mathrm{E}}(\underline{z},\overline{z},\underline{x},\overline{x})\\
       \IFN^{\mathrm{E}}(\overline{z},\underline{z},\overline{x},\underline{x})
      \end{bmatrix}.
  \end{align}
  Theorem~\ref{thm:L4DC}\ref{p1:L4DC} suggests to solve the fixed-point equation~\eqref{eq:fixed-point-IFN} using the $\alpha$-average iterations    $
          \left[\begin{smallmatrix}\underline{z}^{l+1} \\ \overline{z}^{l+1}\end{smallmatrix}\right] = 
          \left[\begin{smallmatrix}\IFN^{\mathrm{E}}_{\alpha}(\underline{z}^l,\overline{z}^l,\underline{x},\overline{x})\\
          \IFN^{\mathrm{E}}_{\alpha}(\overline{z}^l,\underline{z}^l,\overline{x},\underline{x})\end{smallmatrix}\right]$. Since $W^{\FN}$ has a block upper diagonal structure, one can alternatively solve the fixed-point equation~\eqref{eq:fixed-point-IFN} via back-substitution. Using the convention $\underline{z}^*_0=\underline{x}$ and $\overline{z}^*_0=\overline{x}$, we can use back-substitution to obtain
  \begin{align*}
          \underline{z}^*_{i+1}&=\sigma( [W_i]^+ \underline{z}^*_i+[W_i]^- \overline{z}^*_i + b_i)=\FN_i^{\mathrm{E}}(\underline{z}^*_i,\overline{z}^*_i),\\
     \overline{z}^*_{i+1} &= \sigma( [W_i]^+ \overline{z}^*_i+[W_i]^- \underline{z}^*_i + b_i)=\FN_i^{\mathrm{E}}(\overline{z}^*_i,\underline{z}^*_i), 
  \end{align*}
  for every $i\in \{0,1,\ldots,k-1\}$. As a result,
  \begin{align*}
      \underline{\OF}^{\IFN}(\underline{x},\overline{x}) = [C]^+\underline{z}^*_{k} + [C]^-\overline{z}^*_{k} + c =  \underline{\OF}^{\FN}(\underline{x},\overline{x}), \\
      \overline{\OF}^{\IFN}(\underline{x},\overline{x}) = [C]^+\overline{z}^*_{k} + [C]^-\underline{z}^*_{k} + c =  \overline{\OF}^{\FN}(\underline{x},\overline{x}).
  \end{align*}
  This completes the proof of part~\ref{p4:L4DC-FFNN}. 
  \end{proof}



\subsection{Alternative deep modeling}

By separating the notion of depth from the layer-wise evaluation, our mixed monotone contracting approach can be used to estimate the reachable sets of deep weight-tied FFNNs. For the weight-tied FFNN~\eqref{eq:weight-tied}, we define
 \begin{multline*}
         \WFN^{\mathrm{E}}(\underline{z},\overline{z},\underline{x},\overline{x})=\sigma([W]^+ \underline{z}+[W]^-\overline{z} + \\ 
          [U]^{+}\underline{x} + [U]^{-}\overline{x} + b).
      \end{multline*}
      By replacing $\FN_i^{\mathrm{E}}$ with $\WFN^{\mathrm{E}}$ in~\eqref{eq:inclusion:FFNN}, we the IBP inclusion function $\OF^{\WFN}:\mathcal{T}^r\to \real^{2q}$.

\begin{theorem}[Weight-tied infinite-layer neural networks]\label{thm:infinite-layer}
Suppose that $\rho(|W|)<1$ and let $\eta\in \real^n_{>0}$ be the right Perron eignvector of $|W|$. For every $\underline{x}\le x \le \overline{x} \in \real^r$, we recursively define the sequences $\{z^i\}_{i=1}^{\infty}$ and $\{(\underline{z}^i,\overline{z}^i)\}_{i=1}^{\infty}$ starting from $\underline{z}^0\le z^0\le \overline{z}^0$ as follows:
\begin{align*}
z^{i+1} & = \WFN(z^i,x),\\ 
    \begin{bmatrix}\underline{z}^{i+1}\\ \overline{z}^{i+1}\end{bmatrix} &= \begin{bmatrix}\WFN^{\mathrm{E}}(\underline{z}^i,\overline{z}^i,\underline{x},\overline{x})\\ \WFN^{\mathrm{E}}(\overline{z}^i,\underline{z}^i,\overline{x},\underline{x})\end{bmatrix}.
\end{align*}
Then, the following statements hold: 
\begin{enumerate}
    \item\label{p1:infinite-layer} $\lim_{i\to \infty} \begin{bmatrix}\underline{z}^i\\ \overline{z}^i\end{bmatrix} =\begin{bmatrix}\underline{w}^*\\ \overline{w}^*\end{bmatrix}$ for some $\underline{w}^*\le \overline{w}^*\in \real^n$; 
    \item\label{p2:infinite-layer} $\lim_{i\to\infty} z^i = w^*$ for some $w^*\in [\underline{w}^*,\overline{w}^*]$; 
\end{enumerate}
Moreover, for the implicit neural network~\eqref{eq:INN},
\begin{enumerate}\setcounter{enumi}{2}
    \item\label{p4:infinite-layer} $\mu_{\infty,[\eta]^{-1}}(W)<1$;
    \end{enumerate}
If $\OF^{\WFN}$ is the inclusion function by the IBP approach as $k\to \infty$ and $\OF^{\ON}$ is the inclusion function from~\eqref{eq:INN}, then 
    \begin{enumerate}\setcounter{enumi}{3}
    \item\label{p5:infinite-layer} $\underline{\OF}^{\ON}(\underline{x},\overline{x}) \ge  \underline{\OF}^{\WFN}(\underline{x},\overline{x})$ and $ \overline{\OF}^{\ON}(\underline{x},\overline{x}) \le  \overline{\OF}^{\WFN}(\underline{x},\overline{x})$.    
\end{enumerate}
\end{theorem}
\begin{proof}
Regarding part~\ref{p1:infinite-layer}, for every $i\in \mathbb{Z}_{> 0}$, 
\begin{align*}
    &\left\|\begin{bmatrix}\underline{z}^{i+1}\\ \overline{z}^{i+1}\end{bmatrix}-\begin{bmatrix}\underline{z}^{i}\\ \overline{z}^{i}\end{bmatrix}\right\|_{\infty,I_2\otimes [\eta]^{-1}} \\ & = \left\|\begin{bmatrix}\WFN^{\mathrm{E}}(\underline{z}^i,\overline{z}^i,\underline{x},\overline{x}) - \WFN^{\mathrm{E}}(\underline{z}^{i-1},\overline{z}^{i-1},\underline{x},\overline{x})\\ \WFN^{\mathrm{E}}(\overline{z}^i,\underline{z}^i,\overline{x},\underline{x}) - \WFN^{\mathrm{E}}(\overline{z}^{i-1},\underline{z}^{i-1},\underline{x},\overline{x})\end{bmatrix}\right\|_{\infty,I_2\otimes [\eta]^{-1}} \\ & \le \left\|\begin{bmatrix}[W]^+ & [W]^-\\ [W]^- & [W]^+\end{bmatrix}\right\|_{\infty,I_2\otimes [\eta]^{-1}} \left\|\begin{bmatrix}\underline{z}^{i}\\ \overline{z}^{i}\end{bmatrix}-\begin{bmatrix}\underline{z}^{i-1}\\ \overline{z}^{i-1}\end{bmatrix}\right\|_{\infty,I_2\otimes [\eta]^{-1}}
\end{align*}
where the first equality holds by the definition of $\WFN^{\mathrm{E}}$ and the second inequality holds by the fact that $\sigma$ is weakly-increasing. On the other hand, we have
\begin{align*}
    \left\|\begin{bmatrix}[W]^+ & [W]^-\\ [W]^- & [W]^+\end{bmatrix}\right\|_{\infty,I_2\otimes [\eta]^{-1}} &= \|W\|_{\infty,[\eta]^{-1}} \\ & = \rho(|W|) < 1,
\end{align*}
where the first equality holds by the definition of the $[\eta]$-weighted $\infty$-norm and the second inequality holds using the fact that $\eta\in \real^n_{>0}$ is the right Perron eigenvector of $|W|$. The convergence of the sequence $\{(\underline{z}^i,\overline{z}^i)\}_{i=1}^{\infty}$ to some vector $\left[\begin{smallmatrix}\underline{w}^*\\ \overline{w}^*\end{smallmatrix}\right]$ follows from the Banach Contraction Mapping Theorem.

Regarding part~\ref{p2:infinite-layer}, note that by choosing $\underline{x}=\overline{x}=x$ and $\underline{z}^0=z^0=\overline{z}^0$, we get $\underline{z}^i=z^i=\overline{z}^i$, for every $i\in \mathbb{Z}_{>0}$. Therefore, the convergence of the sequence $\{z^i\}_{i=1}^{\infty}$ to some vector $w^*$ follows from part~\ref{p1:infinite-layer}. It remains to show that $\underline{w}^*\le w^*\le \overline{w}^*$. We prove this inequality using induction. First note that by the choice of the initial conditions, we have $\underline{z}^0\le z^0\le \overline{z}^0$. Now we assume that $\underline{z}^i\le z^i\le \overline{z}^i$, for every $i\in \{0,\ldots,m\}$. In turn, we have
\begin{align*}
   \underline{z}^{m+1} & = \sigma([W]^+\underline{z}^{m}+[W]^-\overline{z}^{m}+[U]^+\underline{x}+[U]^-\overline{x}) \\ & \le z^{m+1} = \sigma([W]^+ z^{m}+[W]^- z^{m}+[U]^+\underline{x}+[U]^-\overline{x})  \\ & \le \overline{z}^{m+1} = \sigma([W]^+\overline{z}^{m}+[W]^-\underline{z}^{m}+[U]^+\overline{x}+[U]^-\underline{x}). 
\end{align*}
This complete the proof of induction. Therefore, for every $i\in \mathbb{Z}_{\ge 0}$, we have $\underline{z}^i\le z^i\le \overline{z}^i$. By taking the limit as $i\to \infty$,
\begin{align*}
    \underline{w}^* = \lim_{i\to\infty} \underline{z}^i \le \lim_{i\to\infty} z^i = w^* \le \lim_{i\to\infty} \overline{z}^i = \overline{w}^*. 
\end{align*}

Regarding part~\ref{p4:infinite-layer}, the result holds due to the inequality 
\begin{align*}
    \mu_{\infty,[\eta]^{-1}}(W) \le \|W\|_{\infty,[\eta]^{-1}} = \rho(|W|)<1.  
\end{align*}

Regrading part~\ref{p5:infinite-layer}, first note that by part~\ref{p4:infinite-layer} and Theorem~\ref{thm:L4DC}\ref{p3:L4DC}, the inclusion function $\OF^{\ON} = \left[\begin{smallmatrix}\underline{\OF}^{\ON} \\ \overline{\OF}^{\ON}\end{smallmatrix}\right]$ is well-defined. Moreover, by part~\ref{p1:infinite-layer}, the inclusion function $\OF^{\WFN} = \left[\begin{smallmatrix}\underline{\OF}^{\WFN} \\ \overline{\OF}^{\WFN}\end{smallmatrix}\right]$ is well-defined in the limit as $k\to\infty$. Note that $[W]^+\ge \lceil W \rceil^{\mathrm{Mzl}}$ and $[W]^- \le \lfloor W \rfloor^{\mathrm{Mzl}}$. As a result, starting from $\underline{z}^0\le z^0\le \overline{z}^0$, for every $i\in \mathbb{Z}_{\ge 0}$,
\begin{align*}
    \ON^{\mathrm{E}}(\underline{z}^i,\overline{z}^i,\underline{x},\overline{x}) &\ge \WFN^{\mathrm{E}}(\underline{z}^i,\overline{z}^i,\underline{x},\overline{x}), \\
    \ON^{\mathrm{E}}(\overline{z}^i,\underline{z}^i,\overline{x},\underline{x}) &\le \WFN^{\mathrm{E}}(\overline{z}^i,\underline{z}^i,\overline{x},\underline{x}),
\end{align*}
It is easy to check that $[1-\min_{i\in \{1,\ldots,n\}}[W^{\FN}_{ii}]^{-}]^{-1} = 1$. Therefore, using Theorem~\ref{thm:L4DC}\ref{p1:L4DC}, we get
\begin{align*}
    \lim_{i\to\infty} \begin{bmatrix}
     \ON^{\mathrm{E}}(\underline{z}^i,\overline{z}^i,\underline{x},\overline{x}) \\
     \ON^{\mathrm{E}}(\overline{z}^i,\underline{z}^i,\overline{x},\underline{x})
    \end{bmatrix} = \begin{bmatrix}
     \underline{z}^*\\
     \overline{z}^*
    \end{bmatrix}.
\end{align*}
Then one can use Theorem~\ref{thm:L4DC}\ref{p3:L4DC} to obtain
\begin{align*}
     \underline{\OF}^{\ON}(\underline{x},\overline{x}) &= [C]^+ \underline{z}^* +[C]^-\overline{z}^*+c \\ &=  [C]^+ \lim_{i\to\infty}\ON^{\mathrm{E}}(\underline{z}^i,\overline{z}^i,\underline{x},\overline{x}) \\ & \qquad\qquad\qquad + [C]^- 
    \lim_{i\to\infty}\ON^{\mathrm{E}}(\overline{z}^i,\underline{z}^i,\overline{x},\underline{x}) + c \\ & \ge [C]^+ \lim_{i\to\infty}\WFN^{\mathrm{E}}(\underline{z}^i,\overline{z}^i,\underline{x},\overline{x}) \\ & \qquad\qquad\qquad + [C]^- 
    \lim_{i\to\infty}\WFN^{\mathrm{E}}(\overline{z}^i,\underline{z}^i,\overline{x},\underline{x}) + c \\ &  = \underline{\OF}^{\WFN}(\underline{x},\overline{x}),
\end{align*}
Similarly, one can show that $\overline{\OF}^{\ON}(\underline{x},\overline{x}) \le \overline{\OF}^{\WFN}(\underline{x},\overline{x})$ and this completes the proof of part~\ref{p5:infinite-layer}. 
\end{proof}

\begin{remark}[Comparison with IBP]
 For weight-tied infinite-depth FFNNs, Theorem~\ref{thm:infinite-layer}\ref{p5:infinite-layer} shows that the  mixed monotone contracting approach provides sharper (or equal) estimates of the reachable set than the IBP approach. 
\end{remark}

\section{Training robust implicit neural networks}\label{sec:train}

In this section, we estimate certified adversarial robustness of INNs using the inclusion functions and propose suitable optimization problems for training certifiably robust INNs.


\paragraph*{Certified adversarial robustness for classification tasks}
We say an INN is certifiably robust for input $x$ if its prediction at $x$ is verifiably constant within a given $\ell_\infty$ ball around $x$. We refer to~\cite{SJ-MA-AD-FB-SC:21y} for a rigorous definition of certified adversarial robustness.
We use the embedded INN~\eqref{eq:INN-embedding} to obtain a sufficient condition for certified robustness. Given a robustness radius $\epsilon>0$, for every input $x\in \real^r$, we define 
$\underline{x}=x-\epsilon\vect{1}_r, 
\overline{x}=x+\epsilon\vect{1}_r.$
Following~\cite[Eq. 3]{HZ-etal:20} and~\cite{SJ-MA-AD-FB-SC:21y}, for each input $x' \in [\underline{x},\overline{x}]$, we define the \emph{relative classifier variable}, $m^x(x')\in \real^q$ by
\begin{equation}\label{eq:relativeclassifier}
    m^x(x') = f(x')_i \vectorones[q] - f(x'),
\end{equation} 
where $i$ is the correct label of $x$. Note that $m^x(x')_j > 0$ for all $j \neq i$ if and only if $x'$ is labeled the same as $x$ by the neural network. Therefore, we write $m^x(x') = T^x f(x') = T^x Cz^*(x') + T^x c$, for suitable specification matrix $T^x \in \{-1,0,1\}^{q \times q}$ defined via the linear transformation~\eqref{eq:relativeclassifier}. Moreover, if there exists $\eta\in \real^n_{>0}$ so that $\mu_{\infty,[\eta]^{-1}}(W)<1$, then we can use Theorem~\ref{thm:L4DC} to define
\begin{align*}
    \underline{m}^x(\underline{x},\overline{x}) = [T^x C]^+ \underline{z}^*(\underline{x},\overline{x}) + [T^x C]^- \overline{z}^*(\underline{x},\overline{x}) + T^x c. 
\end{align*}
%
Moreover, $\min_{j \neq i} \{\underline{m}_j^x(\underline{x},\overline{x})\}>0$ is a sufficient condition for certified adversarial robustness of the INN~\cite{SJ-MA-AD-FB-SC:21y}.



\paragraph*{Training optimization problem}
We aim to design optimization problems to train a neural network which is robust to input perturbations with $\ell_{\infty}$-norm smaller than some $\epsilon$. Let $\mathcal{L}$ be the cross-entropy loss function and assume that $\{(\widehat{x}^l,\widehat{y}^l)\}_{l=1}^{N}$ is a set of $N$ labeled data points used for training. For every $l\in \{1,\ldots,N\}$, we define the following upper and the lower bounds on the input $\widehat{x}^l$ by 
    $\underline{x}^l=\widehat{x}^l-\epsilon \vect{1}_r$ and $\overline{x}^l=\widehat{x}^l+\epsilon \vect{1}_r$. 
We use the robust optimization framework~\cite{AM-AM-LS-DT-AV:17} for designing robust neural networks. Our objective is to minimize the robust loss function $\sum_{l=1}^{N} \max_{x\in [\underline{x}^l,\overline{x}^l]}\mathcal{L}(f(\widehat{x}^l),\widehat{y}^l)$ on the training data. However, using the robust loss for training of training neural networks leads to a min-max optimization problem that scales poorly with the size of the training data~\cite{EW-ZK:18}. Using~\cite[Theorem 2]{EW-ZK:18}, for the cross-entropy loss, and for $\underline{m}^l := \underline{m}^{\widehat{x}^l}(\underline{x}^l,\overline{x}^l)$ and every $l\in\{1,\ldots,N\}$,
\begin{align*}
    \mathcal{L}(f(\widehat{x}^l),\widehat{y}^l) \le \mathcal{L}(-\underline{m}^l,\widehat{y}^l),\quad \text{for all } x\in [\underline{x}^l,\overline{x}^l].
\end{align*}
Therefore, one can instead use the loss function $\sum_{l=1}^{N}\mathcal{L}(-\underline{m}^l,\widehat{y}^l)$ as a tractable upper bound on the robust loss in the training optimization problem. 

 \begin{figure*}[!ht]
 \begin{center}
\begin{tabular}{cc}
		\includegraphics[width = 0.44\linewidth,clip]{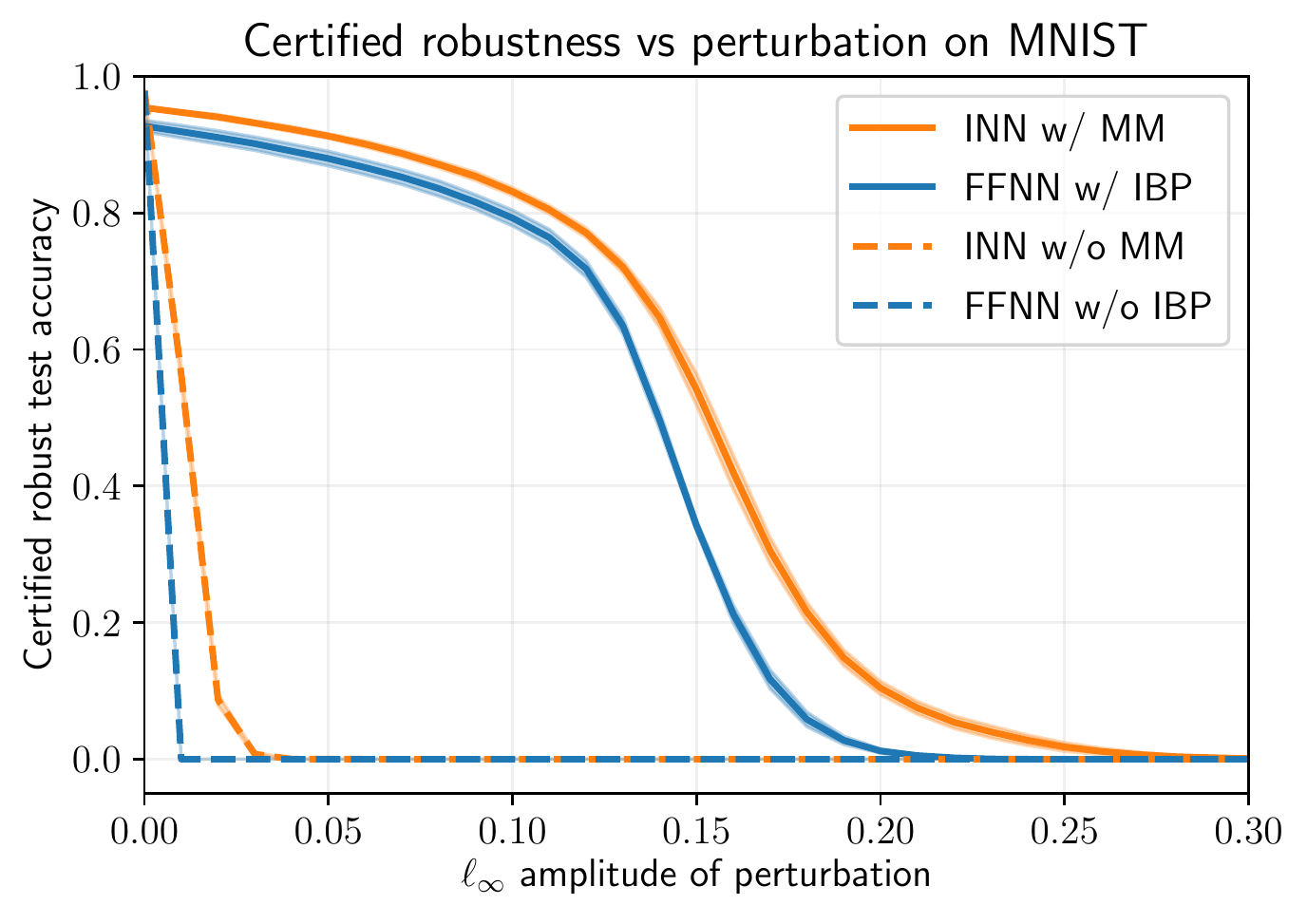}&
		\includegraphics[width = 0.44\linewidth,clip]{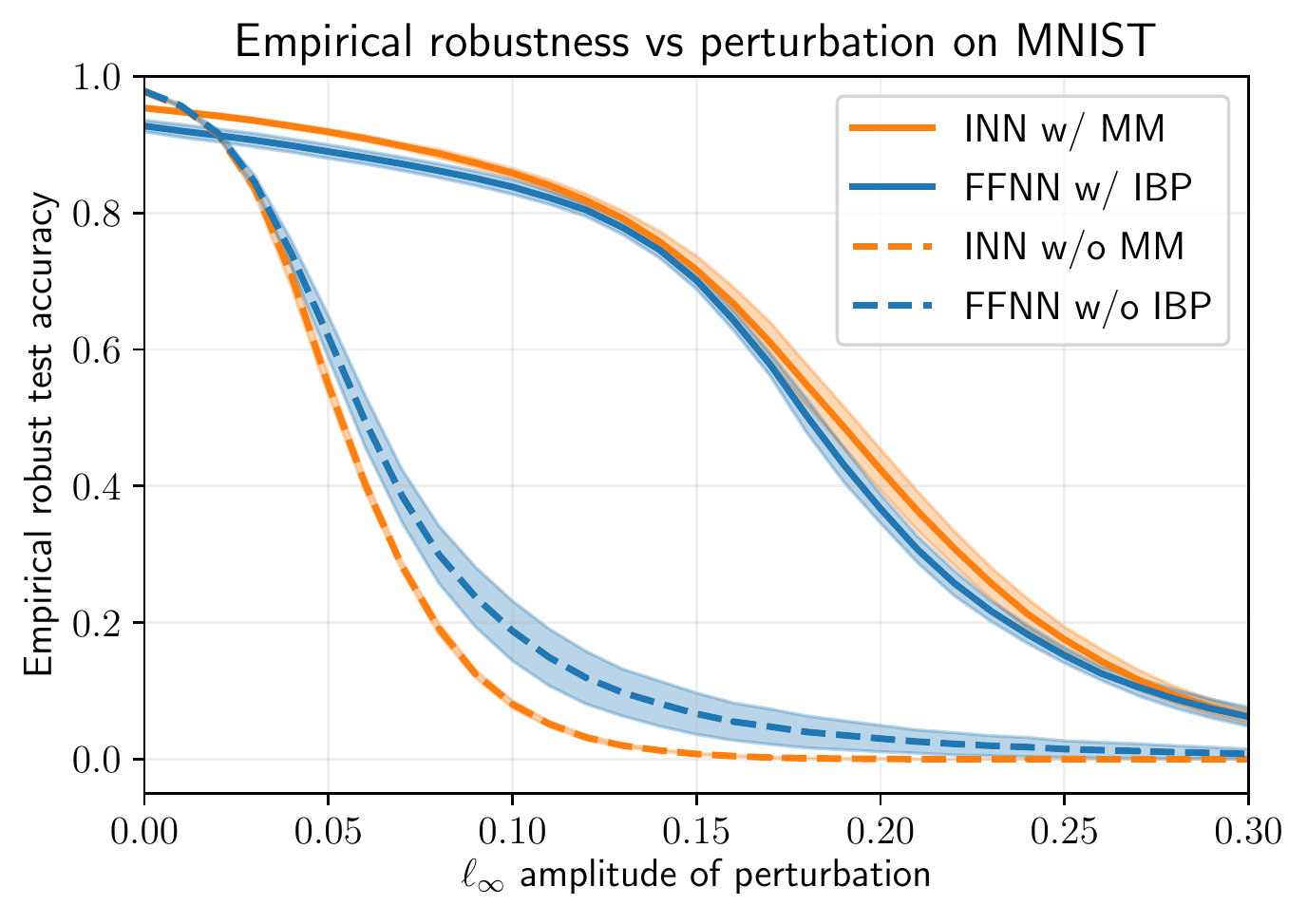}
	\end{tabular}
	\end{center}
	\vspace{-0.35em}
	\caption{Performance comparison on the MNIST test data between INNs trained with and without mixed monotonicity (MM) and $5$-layer FFNNs trained with and without IBP. The INNs have $89710$ trainable parameters and FFNNs have $93200$ trainable parameters. The left plot shows the certified robust accuracy of the models computed using either MM or IBP while the right plot shows the empirical robustness of the models against a PGD attack. In each plot, dark lines correspond to the mean accuracy across $5$ neural networks while light envelopes around the dark lines correspond to one standard deviation.}
	\label{fig:MNIST}
\end{figure*}

As pointed out in~\cite{SG-etal:18} for FFNNs, using the loss function $\sum_{l=1}^{N}\mathcal{L}(-\underline{m}^l,\widehat{y}^l)$ in the training can lead to convergence instability. To improve the stability of the training,  following~\cite{SG-etal:18}, we instead use a convex combination of the empirical risk loss and the robust loss. Therefore, for $T^l := T^{\widehat{x}^l}$ we get the following training problem:
\begin{align}\label{eq:training-robust-2}
     \min_{W,U,C,b,c,\eta}&\quad\quad \sum_{l=1}^{N}(1-\kappa) \mathcal{L}(y^l,\widehat{y}^l) + \kappa \mathcal{L}(-\underline{m}^{l},\widehat{y}^l),\nonumber\\
    \begin{bmatrix}
     \underline{z}^l\\
     \overline{z}^l
    \end{bmatrix} &= \begin{bmatrix}
     \ON^{\mathrm{E}}(\underline{z}^l,\overline{z}^l,\underline{x}^l, \overline{x}^l)\\
     \ON^{\mathrm{E}}(\overline{z}^l,\underline{z}^l,\overline{x}^l,\underline{x}^l)
    \end{bmatrix}, \\ 
     \underline{m}^{l} &=
     [T^{l}C]^+\underline{z}^l+[T^{l}C]^-\overline{z}^l + T^{l}c, \; \; z^l = \ON(z^l,\widehat{x}^l), \nonumber \\ y^l &= Cz^l + c, 
     \quad \mu_{\infty,[\eta]^{-1}}(W)\le \gamma. \nonumber
\end{align}
where $\kappa\in [0,1]$ and $\gamma\in (-\infty,1)$ are hyperparameters.

 \section{Numerical Experiments} 

 In this section we provide an experimental comparison between the robustness of FFNNs and INNs trained with and without IBP and mixed monotonicity, respectively\footnote{Code to reproduce the experiments is available at \url{https://github.com/davydovalexander/robust-inn-mm}.}.
 
 \paragraph*{Experimental setup} We consider the MNIST dataset, which contains $70000$ $28 \times 28$ pixel images of handwritten digits. For training, pixels are normalized into the range $[0,1]$. All INNs have $n = 100$ neurons with ReLU activation, while we consider five-layer FFNNs $(784\to 100\to 75\to 50\to 40\to 25\to 10)$ with ReLU activation. 
 
 Each model was trained for $40$ epochs using the Adam optimizer.
 INNs that were trained using mixed monotonicity and FFNNs that were trained using IBP have $\subscr{\epsilon}{test} = 0.1$ and $\subscr{\kappa}{nom} = 0.75$. From epochs $1$ to $10$, $\kappa$ and $\epsilon$ are set to $0$ so the models undergo regular (nonrobust) training. From epochs $11$ to $20$, $\epsilon$ and $\kappa$ are linearly increased such that at epoch $20$, $\epsilon = \subscr{\epsilon}{test}$ and $\kappa = \subscr{\kappa}{nom}$.
Regarding training INNs, we follow the non-Euclidean monotone operator framework described in~\cite{SJ-AD-AVP-FB:21f}. We impose $\mu_{\infty,[\eta]^{-1}}(W) \leq 0$ for some $\eta \in \realpositive^n$.


$10$ FFNNs and 10 INNs were trained; $5$ of each were trained using IBP or mixed monotonicity (for feedforwad and implicit, respectively) and $5$ of each were trained without any robust optimization (i.e., $\subscr{\epsilon}{test} = 0$). Figure~\ref{fig:MNIST} provides plots of certified adversarial robustness via the corresponding interval reachability technique and the empirically-observed robustness against a projected gradient descent (PGD) attack.

\paragraph*{Evaluation summary} Regarding certified robustness, at an $\ell_\infty$ perturbation radius of $0.1$, we observe that INNs trained using mixed monotonicity had, on average, an accuracy of $83.13\%$, while FFNNs trained using IBP had, on average an accuracy of $79.26\%$. We additionally observe that at the cost of a few percentage points in clean accuracy, both INNs and FFNNs trained robustly vastly outperform non-robustly trained models in both certified and empirical robustness. For example, at an $\ell_\infty$ perturbation radius of $0.1$, INNs trained without mixed monotonicity have an empirical accuracy of $8.04\%$, while INNs trained with mixed monotonicity have an accuracy of $85.84\%$, indicating an order of magnitude improvement in empirical robustness.

\section{Conclusion}
We develop a computationally efficient algorithm for training robust INNs. Moreover, we provide theoretical and empirical evidence in support of the following claims: (i) robustly-trained INNs are more robust than comparably-trained FFNNs, (ii) inclusion functions provide tighter estimates than Lipschitz constants, (iii) robustly-trained networks enjoy much stronger robustness properties than their non-robustly trained counterparts. 


\bibliography{alias,Main,FB,New,SJ}
\bibliographystyle{plain}

\clearpage
\newpage

\end{document}